\documentclass[11pt]{article}
\usepackage{vitercik}
\usepackage{subcaption}
\newsavebox{\imagebox}

	\title{Generalization in portfolio-based algorithm selection}
\author{
	Maria-Florina Balcan \\ \small Carnegie Mellon University \\ \small \texttt{ninamf@cs.cmu.edu}
	\and 
	Tuomas Sandholm \\ \small Carnegie Mellon University \\
	\small Optimized Markets, Inc.\\
	\small Strategic Machine, Inc.\\
	\small Strategy Robot, Inc.\\
	\small \texttt{sandholm@cs.cmu.edu}
	\and 
	Ellen Vitercik \\ \small Carnegie Mellon University \\ \small \texttt{vitercik@cs.cmu.edu}}

\begin{document}
\maketitle

\begin{abstract}
	Portfolio-based algorithm selection has seen tremendous practical success over the past two decades. This algorithm configuration procedure works by first selecting a portfolio of diverse algorithm parameter settings, and then, on a given problem instance, using an \emph{algorithm selector} to choose a parameter setting from the portfolio with strong predicted performance. Oftentimes, both the portfolio and the algorithm selector are chosen using a \emph{training set} of typical problem instances from the application domain at hand. In this paper, we provide the first provable guarantees for portfolio-based algorithm selection. We analyze how large the training set should be to ensure that the resulting algorithm selector's average performance over the training set is close to its future (expected) performance. This involves analyzing three key reasons why these two quantities may diverge: 1) the learning-theoretic complexity of the algorithm selector, 2) the size of the portfolio, and 3) the learning-theoretic complexity of the algorithm's performance as a function of its parameters. We introduce an end-to-end learning-theoretic analysis of the portfolio construction and algorithm selection together. We prove that if the portfolio is large, overfitting is inevitable, even with an extremely simple algorithm selector. With experiments, we illustrate a tradeoff exposed by our theoretical analysis: as we increase the portfolio size, we can hope to include a well-suited parameter setting for every possible problem instance, but it becomes impossible to avoid overfitting. 
\end{abstract}

\section{Introduction}
Algorithms for many problems have tunable parameters. With a deft parameter tuning, these algorithms can often efficiently solve computationally challenging problems. However, the best parameter setting for one problem is rarely optimal for another. \emph{Algorithm portfolios}---which are finite sets of parameter settings---are used in practice to deal with this variability. A portfolio is often used in conjunction with an \emph{algorithm selector}, which is a function that determines which parameter setting in the portfolio to employ on any input problem instance.
Portfolio-based algorithm selection has seen tremendous empirical success, fueling breakthroughs in combinatorial auction winner determination~\citep{Leyton-Brown03:Resource,Sandholm13:Very-Large-Scale}, SAT~\citep{Xu08:Satzilla},  integer programming~\citep{Xu10:Hydra,Kadioglu10:ISAC}, planning~\citep{Cenamor16:Ibacop,Nunez15:Automatic}, and many other domains.

Both the portfolio and the algorithm selector are often chosen using a \emph{training set} of problem instances from the application domain at hand. This training set is typically assumed to be drawn from an unknown, application-specific distribution. The portfolio and algorithm selector are chosen to have strong average performance (quantified by low average runtime, for example) over the training set. We investigate whether the learned algorithm selector also has strong expected performance on problems from the same application domain. The difference between average performance and expected performance is known as \emph{generalization error}. If the generalization error is small, every parameter setting’s average performance over the training set is close to its expected performance, so the learned algorithm selector will not \emph{overfit}. When overfitting occurs, the learned selector has strong average performance over the training set but poor expected performance on the true distribution. In other words, the algorithm selector is overfitting to the problem instances in the training set.

There are multiple reasons the generalization error might be large in this setting: 1) the learning-theoretic complexity of the algorithm selector, 2) the size of the portfolio, and 3) the learning-theoretic complexity of the algorithm's performance as a function of its parameters. We provide end-to-end bounds on generalization error in terms of all three elements simultaneously. The variety of factors impacting generalization error differentiates this paper from prior research on generalization guarantees in algorithm configuration~\citep{Balcan19:How,Liu20:Performance,Gupta17:PAC,Balcan17:Learning,Balcan18:Learning,Balcan18:Dispersion,Balcan20:LearningToLink,Balcan20:Learning, Balcan20:Refined,Balcan18:Data,Garg18:Supervising}. That research focuses on bounding the generalization error of learning a \emph{single} good parameter setting for the entire problem instance distribution, rather than a portfolio together with an algorithm selector that selects an algorithm (e.g., its parameter values) from the portfolio \emph{for the specific instance at hand}. In the former case, generalization error only grows with (3)---just one of the sources of error we must contend with.

Our bounds apply to the widely-applicable setting where on any fixed input, algorithmic performance is a piecewise-constant function of its parameters with at most $t$ pieces, for some $t \in \Z$. This structure has been observed in algorithm configuration for integer programming, greedy algorithms, clustering, and computational biology~\citep{Gupta17:PAC,Balcan17:Learning,Balcan18:Learning,Balcan19:How,Balcan20:Data}. Given a training set of size $N$, we prove that the generalization error is bounded\footnote{Here we assume that algorithmic performance is a quantity in $[0,1]$, an assumption we relax in Section~\ref{sec:formulation}.} by $\tilde{O}\left(\sqrt{\left(\bar{d} + \kappa \log t\right)/N}\right),$ where $\kappa$ is the size of the portfolio and $\bar{d}$ measures the \emph{intrinsic complexity} of the algorithm selector, as we define in Section~\ref{sec:sample}. We also prove that this bound is tight up to logarithmic factors: the generalization error can be as large as $\tilde{\Omega}\left(\sqrt{\left(\bar{d} + \kappa\right)/N}\right).$ This implies that even if the algorithm selector is extremely simple ($\bar{d}$ is small), overfitting cannot be avoided in the worst case when the portfolio size $\kappa$ is large. Moreover, we instantiate our guarantees for several commonly-used families of algorithm selectors~\citep{Xu08:Satzilla,Kadioglu10:ISAC,Hutter14:Algorithm}. 

Finally, via experiments in the context of integer programming configuration, we illustrate the inherent tradeoff our theory exposes: as we increase the portfolio size, we can hope to include a high-performing parameter setting for any given instance, but it become increasingly difficult to avoid overfitting. We incrementally increase the size of the portfolio and with each addition we train an algorithm selector using regression forest performance models. As the portfolio size increases, the algorithm selector's training performance continues to improve, but there comes a point where the test performance begins to worsen, meaning that the algorithm selector is overfitting to the training set.

\paragraph{Additional related research.}
\citet{Gupta17:PAC} also provide generalization guarantees for algorithm configuration. They primarily analyze the problem of learning a single parameter setting with high expected performance on the underlying distribution. They do provide guarantees for the more general problem of learning a mapping from instances to parameter settings in a few special cases, but do not study the problem of learning a portfolio in conjunction with learning a selector, which we do. They study settings where for each problem instance, a domain expert has defined a number of relevant features, as do we in Section~\ref{sec:applications}. Their first result applies to learning an algorithm selector when the set of features is finite. In contrast, our results apply to infinite feature spaces. Their second set of results is tailored to the problem of learning empirical performance models and applies when the feature space is infinite. An empirical performance model is meant to predict how long a particular algorithm will take to run on a given input. An algorithm selector can use an empirical performance model by selecting the parameter setting with best predicted performance. \citet{Gupta17:PAC} provide guarantees that bound the difference between the empirical performance model's expected error and average error over the training set. Their guarantees can be applied once the portfolio is already chosen. They do not study the problem of learning the portfolio itself, whereas we study the composite problem of learning the portfolio and the algorithm selector.

In a related theoretical direction, several papers have studied a model where there are multiple algorithms capable of computing a correct solution to a given problem, but with different costs. The user can run multiple algorithms until one terminates with the correct solution. Given a training set of problem instances, the authors
provide guarantees for learning a schedule with high expected performance~\citep{Sayag06:Combining,Streeter07:Combining,Streeter09:Online}. That is a distinct problem from ours, since our goal is to learn an algorithm selector rather than a schedule. Moreover, we additionally handle the problem of learning the portfolio itself. 

\section{Problem formulation and road map}\label{sec:formulation}
\paragraph{Notation.} Our theoretical guarantees apply to algorithms
parameterized by a real value $\rho \in \R$.
We use the notation $\cZ$ to denote the set of problem instances the algorithm may take as input. For example, $\cZ$ might consist of integer programs (IPs) if we are configuring an IP solver. There is an unknown distribution $\dist$ over problem instances in $\cZ$.

To describe the performance of a parameterized algorithm, we adopt the notation of prior research~\citep{Balcan19:How}.
For every parameter setting $\rho \in \R$, there is a function $u_{\rho} : \cZ \to [0,H]$ that measures, abstractly, the performance of the algorithm parameterized by $\rho$ given an input $z \in \cZ$. For example, $u_{\rho}$ might measure runtime or the quality of the algorithm's output. We use the notation $\cU = \left\{u_{\rho} : \rho \in \R\right\}$ to denote the set of all performance functions. 
\paragraph{Problem formulation.}
A portfolio-based algorithm selection procedure relies on two key components: a \emph{portfolio} and an \emph{algorithm selector}. A portfolio is a set $\cP = \left\{\rho_1, \dots, \rho_{\kappa}\right\} \subseteq \R$ of $\kappa$ parameter settings. An algorithm selector is a mapping
$f: \cZ \to \cP$ from problem instances $z \in \cZ$ to parameter settings $f(z) \in \cP$.
In practice~\citep{Xu10:Hydra,Kadioglu10:ISAC,Sandholm13:Very-Large-Scale}, the portfolio  and algorithm selector are typically learned using the following high-level procedure:
\begin{enumerate}
	\item Choose a class $\cF$ of algorithm selectors, each of which maps $\cZ$ to $\R$. (In Section~\ref{sec:applications}, we provide several examples of classes $\cF$ used in practice.)
	\item Draw a training set $\sample = \left\{z_1, \dots, z_N\right\} \sim \dist^N$ of problem instances from the unknown distribution $\dist$.
	\item Use $\sample$ to learn a portfolio $\hat{\cP} = \left\{\rho_1, \dots, \rho_{\kappa}\right\} \subseteq \R$.
	\item Use $\sample$ to learn an algorithm selector $\hat{f} \in \cF$ that maps to parameter settings in the portfolio $\hat{\cP}$.
\end{enumerate}

Given an instance $z \in \cZ$, the performance of the parameter setting selected by $\hat{f}$ is $u_{\hat{f}(z)}(z)$.
 We bound the expected quality $\E_{z \sim \dist}\left[u_{\hat{f}(z)}(z)\right]$ of the learned algorithm selector.

\paragraph{Road map.}
We first analyze to what extent the average performance of the selector $\hat{f}$ over the training set generalizes to its expected performance on the distribution. 
We then use this analysis to relate the performance of the learned selector $\hat{f}$ and the optimal selector under the optimal choice of a portfolio. In particular, we bound the difference between  $\E_{z \sim \dist}\left[u_{\hat{f}(z)}(z)\right]$ and $\max_{\cP : |\cP| \leq \kappa}\E_{z \sim \dist}\left[\max_{\rho \in \cP} u_{\rho}(z) \right]$. (Equivalently, if our goal is to minimize $u_{\rho}(z)$, we may replace each $\max$ with a $\min$.) 

\section{Sample complexity bounds}\label{sec:sample}
In this section, we bound the difference between the average performance of any selector $f \in \cF$ over the training set $\sample \sim \dist^N$ and its expected performance. Formally, we bound \begin{equation}\left|\frac{1}{N} \sum_{z \in \sample} u_{f(z)}(z) - \E_{z \sim \dist}\left[u_{f(z)}(z)\right]\right|\label{eq:generalization}\end{equation} for any choice of an algorithm selector $f \in \cF$. This will serve as a building block for our general analysis of portfolio-based algorithm selection.

Our bounds apply in the widely-applicable setting where on any fixed input, algorithmic performance is a piecewise-constant function of the algorithm's parameters. This structure has been observed in algorithm configuration for integer programming, greedy algorithms, clustering, and computational biology~\citep{Gupta17:PAC,Balcan17:Learning,Balcan18:Learning,Balcan19:How,Balcan20:Data}. To describe this structure more formally, for a fixed input $z \in \cZ$, we use the notation $u_z^* : \R\to \R$ to denote algorithmic performance as a function of the parameters (whereas the functions $u_{\rho}$ defined in Section~\ref{sec:formulation} measure performance as a function of the input $z$). Naturally, $u_z^*(\rho) = u_{\rho}(z)$. We refer to $u_z^*$ as a \emph{dual function} (as opposed to $u_{\rho}$, which is a \emph{primal function}). We assume algorithmic performance is a piecewise-constant function of the parameters, or more formally, that each function $u_z^*$ is piecewise constant with at most $t$ pieces, for some $t \in \Z$.

Our bounds depend on both the number of pieces $t$ and on the \emph{intrinsic complexity} of the class of algorithm selectors $\cF$.
We use the following notion of the \emph{multi-class projection} of $\cF$ to define the class's intrinsic complexity.
\begin{definition}
	Given a selector $f \in \cF$, let $\rho_1 < \rho_2 < \cdots < \rho_{\bar{\kappa}}$ be the parameter settings $f$ maps to, with $\bar{\kappa} \leq \kappa$. The function $f$ defines a partition $Z_1, \dots, Z_{\bar{\kappa}}$ of the problem instances $\cZ$ where for any $z \in \cZ$, if $f(z)  = \rho_i$, then $z \in Z_i$. For each function $f \in \cF$ there is therefore a corresponding multi-class function $\bar{f} : \cZ \to [\kappa]$ that indicates which set of the partition the instance $z$ belongs to: $\bar{f}(z) = i$ when $z \in Z_i$. We use the notation $\bar{\cF} = \left\{ \bar{f} : f \in \cF\right\}$ to denote the set of all such multi-class functions.
\end{definition}
Defining this set of multi-class functions allows us to use classic tools from multi-class learning to reason about the algorithm selectors $\cF$.
In particular, our bounds depend on the \emph{\citet{Natarajan89:Learning} dimension} of the class $\bar{\cF}$, which is a natural extension of the classic VC dimension~\citep{Vapnik71:Uniform} to multi-class functions.
	\begin{definition}[Natarajan dimension]
		The set $\bar{\cF}$ \emph{multi-class shatters} a set of problem instances $z_1, \dots, z_N$ if there exist labels $y_1, \dots, y_N \in [\kappa]$ and $y_1', \dots, y_N' \in [\kappa]$ such that:
		\begin{enumerate}
			\item For every $i \in [N]$, $y_i \not= y_i'$, and
			\item For any subset $C \subseteq [N]$, there exists a function $\bar{f} \in \bar{\cF}$ such that $\bar{f}\left(z_i\right) = y_i$ if $i \in C$ and $\bar{f}(z_i) = y_i'$ otherwise.
		\end{enumerate}The \emph{Natarajan dimension} of $\bar{\cF}$ is the cardinality of the largest set that can be multi-class shattered by $\bar{\cF}$.
	\end{definition}
In Section~\ref{sec:applications}, we bound the Natarajan dimension of $\bar{\cF}$ for several commonly-used classes of algorithm selectors $\cF$.
We use Natarajan dimension to quantify the intrinsic complexity of the class of selectors, which in turn allows us to bound Equation~\eqref{eq:generalization} for every function $f \in \cF$. To do so, we relate the Natarajan dimension of $\bar{\cF}$ to the
 \emph{pseudo-dimension} of the function class $\cU_{\cF} = \left\{z \mapsto u_{f(z)}(z) : f \in \cF\right\}$. Every function in $\cU_{\cF}$ is defined by an algorithm selector $f \in \cF$. On input $z \in \cZ$, $u_{f(z)}(z)$ equals the utility of the algorithm parameterized by $f(z)$ on input $z$.
Pseudo-dimension~\citep{Haussler92:Decision} is a classic learning-theoretic tool for measuring the intrinsic complexity of a class of real-valued functions (whereas Natarajan dimension applies to multi-class functions). Both Natarjan dimension and pseudo-dimension are extensions of the classic VC dimension, so they bear some resemblance. Below, we define the pseudo-dimension of the class $\cU_{\cF}$.
\begin{definition}[Pseudo-dimension]
	The set $\cU_{\cF}$ \emph{shatters} a set of instances $z_1, \dots, z_N \in \cZ$ if there exist \emph{witnesses} $w_1, \dots, w_N \in \R$ such that for any subset $C \subseteq [N]$, there exists an algorithm selector $f \in \cF$ such that $u_{f\left(z_i\right)}\left(z_i\right) \leq w_i$ if $i \in C$ and $u_{f\left(z_i\right)}\left(z_i\right) > w_i$ otherwise. The \emph{pseudo-dimension} of $\cU_{\cF}$, denoted $\pdim\left(\cU_{\cF}\right)$, is the size of the largest set of instances that can be shattered by $\cU_{\cF}$.
\end{definition}

Classic learning-theoretic results allow us to provide generalization bounds once we calculate the  pseudo-dimension. For example~\citep{Haussler92:Decision}, with probability $1-\delta$ over the draw of the set $\left\{z_1, \dots, z_N\right\} \sim \dist^N$, for any selector $f \in \cF$, \begin{equation}
	\left|\frac{1}{N} \sum_{i = 1}^N u_{f(z_i)}(z_i)- \E_{z \sim \dist}\left[u_{f(z)}(z)\right]\right|
	= O\left(H\sqrt{\frac{1}{N} \left(\pdim \left(\cU_{\cF}\right) + \log \frac{1}{\delta}\right)}\right).\label{eq:haussler}
\end{equation}

We now prove a general bound on $\pdim \left(\cU_{\cF}\right) $, which allows us to bound Equation~\eqref{eq:generalization}. The proof is in Appendix~\ref{app:sample}.
\begin{restatable}{theorem}{nat}\label{thm:nat}
	Suppose each dual function $u_z^*$ is piecewise-constant with at most $t$ pieces.
	Let $\bar{d}$ be the Natarajan dimension of $\bar{\cF}$.
	Then $\pdim \left(\cU_{\cF}\right) = \tilde O\left(\bar{d} + \kappa\log t\right).$
\end{restatable}

At a high level, the $\tilde O\left(\bar{d}\right)$ term accounts for the intrinsic complexity of the algorithm selectors $\cF$. The $O\left(\kappa\log t\right)$ term accounts for the complexity of composing selectors $f$ with the performance functions $u_{\rho}$. In Theorem~\ref{thm:lb}, we prove this bound is tight up to logarithmic factors.

\begin{proof}[Proof sketch of Theorem~\ref{thm:nat}]
Let $z_1, \dots, z_N \in \cZ$ be an arbitrary set of problem instances. Since each dual function $u_{z_i}^*$ is piecewise-constant with at most $t$ pieces, there are $M \leq Nt$ intervals $I_1, \dots, I_M$ partitioning $\R$ where for any interval $I_j$ and any instance $z_i$, $u_{z_i}^*(\rho)$ is constant across all $\rho \in I_j$.
Given these intervals, we partition the algorithm selectors in $\cF$ into at most $M^{\kappa}$ sets so that within any one set, all selectors map to the same $\kappa$ (or fewer) intervals. Focusing on the selectors within one set $\cF_0$ of the partition, we prove that the number of ways the utility functions $u_f$ across $f \in \cF_0$ can labels the instances $z_1, \dots, z_N$ is upper bounded by the number of ways the multi-class projection functions $\bar{f}$ across $f \in \cF_0$ can label the instances. We can then use the Natarajan dimension of $\bar{\cF}$ to bound the number of ways the functions in $\cU_{\cF}$ label the instances $z_1, \dots, z_N$.
\end{proof}

Theorem~\ref{thm:nat} and Equation~\eqref{eq:haussler} imply that with probability $1-\delta$ over the draw $\sample \sim \dist^N$, for any selector $f \in \cF$, \begin{equation}
	\left|\frac{1}{N} \sum_{z \in \sample} u_{f(z)}(z) - \E_{z \sim \dist}\left[u_{f(z)}(z)\right]\right| = O\left(H\sqrt{\frac{1}{N}\left(\bar{d} + \kappa\log t + \log\frac{1}{\delta}\right)}\right).\label{eq:ub}\end{equation} This theorem quantifies a fundamental tradeoff: as the portfolio size increases, we can hope to obtain better and better empirical performance $\sum_{z \in \sample} u_{f(z)}(z)$ but the generalization error $\tilde O\left(H\sqrt{\left(\bar{d} + \kappa\right)/N}\right)$ will worsen.

We now prove that Theorem~\ref{thm:nat} is tight up to logarithmic factors. The following theorem illustrates that even if the class of algorithm selectors is extremely simple (in that the Natarajan dimension of $\bar{\cF}$ is 0), if the portfolio size (that is, the number $\kappa$ of parameters mapped to) is large, we cannot hope to avoid overfitting. The full proof is in Appendix~\ref{app:sample}.

\begin{restatable}{theorem}{lb}\label{thm:lb}
	For any $\kappa, \bar{d} \geq 2$, there is a class of functions $\cU = \left\{u_{\rho}: \rho \in \R\right\}$ and a class of selectors $\cF$ such that:
	\begin{enumerate}
		\item Each selector $f \in \cF$ maps to $\leq \kappa$ parameter settings.
		\item Each dual function $u_z^*$ is piecewise-constant with 1 discontinuity,
		\item The Natarajan dimension of $\bar{\cF}$ is at most $\bar{d}$, and
		\item The pseudo-dimension  of $\cU_{\cF}$ is $\Omega\left(\kappa + \bar{d}\right)$.
	\end{enumerate}
\end{restatable}
\begin{proof}[Proof sketch]
	Let $\cZ = (0,1]$. For each parameter setting $\rho \in \R$, define $u_{\rho}(z) = \textbf{1}_{\{z \leq \rho\}}$.
Let $\kappa, \bar{d} \geq 2$ be two arbitrary integers. We split this proof into two cases: $\bar{d} \geq \kappa$ and $\kappa > \bar{d}$. In both cases, we construct a class of selectors $\cF$ that satisfies the properties in the theorem statement and we prove that $\pdim\left(\cU_{\cF}\right) \geq \max\left\{\kappa, \bar{d}\right\} = \Omega\left(\kappa + \bar{d}\right)$. We sketch the proof of the case where $\kappa > \bar{d}$.

We begin by partitioning $\cZ = (0,1]$ into $\kappa$ intervals $Z_1, \dots, Z_{\kappa}$, where $Z_i = \left(\frac{i-1}{\kappa}, \frac{i}{\kappa}\right]$. For each set $C \subseteq [\kappa]$, we define an selector $f_{C} : \cZ \to \R$ as follows.
For any $z \in \cZ$, let $i$ be the index of the interval $z$ lies in, i.e., $z \in Z_i$. If $i \in C$, we map $f_C(z) = \frac{i}{\kappa}$ and if $i \not\in C$, we map $f_C(z) = \frac{i}{\kappa} - \frac{1}{2\kappa}$.
Let $\cF= \left\{f_{C} : C \subseteq [\kappa]\right\}$. The multi-class projection of $\bar{\cF}$ is extremely simple: its Natarjan dimension is $0$. Moreover, the set $\sample = \left\{\frac{1}{\kappa}, \frac{2}{\kappa}, \dots, \frac{\kappa - 1}{\kappa}, 1\right\}$ is shattered by $\cU_{\cF}$ because---at a high level---each selector $f_C$ maps each element $z \in \sample$ to a parameter just above $z$ or just below $z$, which allows the function class $\cU_{\cF}$ to shatter $\sample$.
\end{proof}

In the proof of Theorem~\ref{thm:lb}, each performance function $u_{\rho}$ maps to $\{0,1\}$, so we effectively prove a lower bound on the \emph{VC dimension} of $\cU_{\cF}$. Classic results from learning theory imply the generalization error of learning a selector $f \in \cF$ can therefore be as large as $\tilde{\Omega}\left(H\sqrt{\left(\bar{d} + \kappa\right)/N}\right),$ which matches Equation~\eqref{eq:ub} up to logarithmic factors.

\section{Application of theory to algorithm selectors}\label{sec:applications}
We now instantiate Theorem~\ref{thm:nat} for several commonly-used classes of algorithm selectors. In each of the case studies, there is a feature mapping $\phi: \cZ \to \R^m$ that assigns feature vectors $\phi(z) \in \R^m$ to problem instances $z \in \cZ$.

\subsection{Linear performance models}\label{sec:linear}
We begin by providing guarantees for algorithm selectors that use a linear performance model. These have been used extensively in computational research~\citep{Xu08:Satzilla,Xu10:Hydra}.
To define this type of  selector, let $\vec{\rho} = \left(\rho_1, \dots \rho_{\kappa}\right)$ be a set of $\kappa$ distinct parameter settings.
For each $i \in [\kappa]$, define a vector $\vec{w}_i \in \R^m$ and let \[W = \begin{pmatrix}\vline & \hdots & \vline\\
		\vec{w}_1 & \ddots & \vec{w}_{\kappa}\\
		\vline & \hdots & \vline\end{pmatrix}\] be a matrix containing all $\kappa$ weight vectors. The dot product $\vec{w}_i \cdot \phi(z)$ is meant to estimate the performance of the algorithm parameterized by $\rho_i$ on  instance $z$.
We define the algorithm selector $f_{\vec{\rho}, W}(z) = \rho_i$ where $i = \argmax_{j \in [\kappa]}\left\{\vec{w}_j\cdot \phi(z)\right\}$, which selects the parameter setting with best predicted performance. We define the class of algorithm selectors $\cF_L = \left\{f_{\vec{\rho}, W} : W \in \R^{m \times \kappa}, \vec{\rho} \in \R^{\kappa}\right\}$.
To define the class $\bar{\cF}_L$, for each matrix $W\in \R^{m \times \kappa}$, let $g_W : \cZ \to [\kappa]$ be a function where $g_W(z) = \argmax_{i \in [\kappa]}\left\{\vec{w}_i\cdot \phi(z)\right\}.$ By definition, $\bar{\cF}_L= \left\{g_W : W \in \R^{m \times \kappa}\right\}$, so $\bar{\cF}_L$ is the well-studied \emph{$m$-dimensional linear class} which has a Natarajan dimension of $O(m\kappa)$~\citep{Shalev14:Understanding}. This fact implies the following corollary.
	\begin{cor}
		Suppose the dual functions are piecewise-constant with at most $t$ pieces. The pseudo-dimension of $\cU_{\cF_L} = \left\{z \mapsto u_{f(z)} : f \in \cF_L\right\}$ is $O(\kappa m \log (\kappa m) + \kappa \log t)$.
	\end{cor}

\subsection{Regression tree performance models}\label{sec:tree}
We now analyze algorithm selectors that use a regression tree as the performance model. These have proven powerful in computational research~\citep{Hutter14:Algorithm}. A regression tree $T$'s leaf nodes partition the feature space $\R^m$ into disjoint regions $R_1, \dots, R_\ell$. In each region $R_i$, a constant value $c_i$ is used to predict the algorithm's performance on instances in the region. The internal nodes of the tree define this partition: each performs an inequality test on some feature of the input. We use the notation $h_T(z)$ to denote tree $T$'s prediction of the algorithm's performance on instance $z$. Formally, $h_T(z)$ equals the constant value corresponding to the region of the tree's partition to which $\phi(z)$ belongs.

An algorithm selector can be defined using a regression tree performance model as follows. Let $\vec{\rho} = \left(\rho_1, \dots, \rho_{\kappa}\right)$ be a set of $\kappa$ distinct parameter settings. For each parameter setting $\rho_i$, let $T_i$ be a tree that is meant to predict the performance of the algorithm parameterized by $\rho_i$, and let $\vec{T} = \left(T_1, \dots, T_{\kappa}\right)$ be the set of all $\kappa$ trees. We define the algorithm selector $f_{\vec{\rho}, \vec{T}}(z) = \rho_i$ where $i = \argmax_{j \in [\kappa]}\left\{h_{T_j}(z)\right\}$. The class of algorithm selectors $\cF_R$ consists of all functions $f_{\vec{\rho}, \vec{T}}$ across all parameter vectors $\vec{\rho} \in \R^{\kappa}$ and all $\kappa$-tuples of regression trees $\vec{T} = \left(T_1, \dots, T_{\kappa}\right)$. The full proof of the following lemma is in Appendix~\ref{app:rt}.

\begin{restatable}{lemma}{rt}
Suppose we limit ourselves to building regression trees with at most $\ell$ leaves.	Then the Natarajan dimension of $\bar{\cF}_R$ is $O(\ell\kappa \log (\ell \kappa m))$.
	\end{restatable}
\begin{proof}[Proof sketch]
	For each $\kappa$-tuple of regression trees $\vec{T} = \left(T_1, \dots, T_{\kappa}\right)$, let $g_{\vec{T}} : \cZ \to [\kappa]$ be a function where $g_{\vec{T}}(z) = \argmax_{i \in [\kappa]} \left\{h_{T_i}(z)\right\}$. By definition, the set $\bar{\cF}_R$ consists of the functions $g_{\vec{T}}$ across all $\kappa$-tuples of regression trees $\vec{T}$ with at most $\ell$ leaves.
	Let $z_1, \dots, z_N \in \cZ$ be a set of problem instances. Our goal is to bound the number of ways the functions $g_{\vec{T}}$ can label these instances. A single regression tree induces a partition of these $N$ problem instances defined by which leaf each instance is mapped to as we apply the tree's inequality tests. The key step in this proof is bounding the total number of partitions we can induce by varying the tree's inequality tests. We then generalize this intuition to bound the number of partitions $\kappa$ regression trees can induce as we vary all their parameters. Once the partition of each regression tree is fixed, the tree with the largest prediction for each problem instance depends on the relative ordering of the constants at the trees' leaves. There is a bounded number of possible relative orderings, and we aggregate all of these bounds to prove the lemma statement.
\end{proof}

	\begin{cor}
	Suppose the dual functions are piecewise-constant with at most $t$ pieces and we limit ourselves to building regression trees with at most $\ell$ leaves. Then $\pdim\left(\cU_{\cF_R}\right) = O(\ell\kappa \log (\ell \kappa m) + \kappa \log t)$.
\end{cor}

This pseudo-dimension bound reflects the end-to-end nature of our analysis, since the guarantee bounds the generalization error of both selecting the portfolio and training the regression tree performance model. This is why the bound grows with both the size of the portfolio $(\kappa)$ and the complexity of the regression trees ($\ell$ and $m$).

\subsection{Clustering-based algorithm selectors}\label{sec:cluster}
We now provide guarantees for clustering-based algorithm selectors, which have also been used in computational research~\citep{Kadioglu10:ISAC}. 
This type of selector clusters the feature vectors $\phi(z_1), \dots, \phi(z_N) \in \R^m$ and chooses a good parameter setting for each cluster.
On a new instance $z$, the selector determines which cluster center is closest to $\phi(z)$ and runs the algorithm using the parameter setting assigned to that cluster.
More formally, let $\vec{\rho} = (\rho_1, \dots, \rho_{\kappa})$ be a set of parameter settings and let $\vec{x}_1, \dots, \vec{x}_{\kappa} \in \R^m$ be a set of vectors. We define the matrix \[X = \begin{pmatrix}\vline & \hdots & \vline\\
			\vec{x}_1 & \ddots & \vec{x}_{\kappa}\\
			\vline & \hdots & \vline\end{pmatrix},\] where each column $\vec{x}_i$ is meant to represent a cluster center.
We define the algorithm selector $f_{\vec{\rho}, X}(z)  = \rho_i$ where $i = \argmin_{j \in [\kappa]}\left\{\norm{\vec{x}_j - \phi(z)}_p\right\},$ for some $\ell_{p}$-norm with $p \geq 1$. The class of algorithm selectors is $\cF_C = \left\{f_{\vec{\rho}, X} : \vec{\rho}\in \R^{\kappa}, X \in \R^{m \times \kappa}\right\}.$ The full proof of the following lemma is in Appendix~\ref{app:cluster}.

\begin{restatable}{lemma}{cluster}\label{lem:cluster}
	For any $p \in [1, \infty)$, the Natarajan dimension of  $\bar{\cF}_C$ is $O(m\kappa \log (m\kappa p))$.
\end{restatable}
\begin{proof}[Proof sketch]
For each matrix $X$, let $g_X : \cZ \to [\kappa]$ be defined such that \[g_X(z) = \argmin_{i \in [\kappa]}\left\{\norm{\vec{x}_i - \phi(z)}^p_p\right\}.\]
	By definition, $\bar{\cF}_C= \left\{g_X : X \in \R^{m \times \kappa}\right\}$.
	Let $z_1, \dots, z_N \in \cZ$ be a set of problem instances. Our goal is to bound the number of ways the functions $g_{X}$ can label these instances as we vary $X \in \R^{m \times \kappa}$. We do so by analyzing, for each instance $z_i$, the boundaries in $\R^{m \times \kappa}$ where if we shift $X$ from one side of the boundary to the other, the column in $X$ closest to $\phi\left(z_i\right)$ changes. We show that these boundaries are defined by multi-dimensional polynomials. We bound the total number of regions these boundaries induce in $\R^{m \times \kappa}$, which implies a bound on the Natarajan dimension of $\bar{\cF}_C$.
\end{proof}

Lemma~\ref{lem:cluster} and Theorem~\ref{thm:nat} imply the following bound.

\begin{cor}\label{cor:cluster}
	Suppose the dual functions are piecewise-constant with at most $t$ pieces. Then $\pdim\left(\cU_{\cF_C}\right)= \tilde O\left(m\kappa + \kappa\log t\right)$.
\end{cor}

\section{Learning procedure with guarantees}\label{sec:procedure}
In this section, we use the results from the previous section to provide guarantees for the high-level learning procedure outlined in Section~\ref{sec:formulation}:
\begin{enumerate}
	\item Draw a training set of problem instances $\sample \sim \dist^N$.
	\item Use the training set $\sample$ to select a set of $\kappa$ or fewer parameter settings $\hat{\cP} \subseteq \R$.
	\item Use $\sample$ to learn an algorithm selector $\hat{f} \in \cF$ that maps problem instances $z \in \cZ$ to parameter settings $\hat{f}(z) \in \hat{\cP}$.
\end{enumerate}

Our guarantees depend on the quality of the portfolio $\hat{\cP}$ and selector $\hat{f}$, as formalized by the following definition.

\begin{definition}\label{def:opt}
	Given a training set $\sample \subseteq \cZ^N$ and parameters $\alpha \in (0,1]$, $\beta \in [0,1]$, and $\epsilon \in [0,1]$, we say the portfolio $\hat{\cP}$ and the algorithm selector $\hat{f}$ are $(\alpha, \beta, \epsilon)$-optimal if:
	\begin{enumerate}
		\item The portfolio $\hat{\cP}$ is nearly optimal over the training set in the sense that \[\frac{1}{N}\sum_{z \in \sample} \max_{\rho \in \hat{\cP}} u_{\rho}(z) \geq \alpha\max_{\cP \subset \R : |\cP| \leq \kappa}\frac{1}{N}\sum_{z \in \sample} \max_{\rho \in \cP} u_{\rho}(z) - \beta.\]
(The maximization means that performance is measured with respect to an oracle that selects an optimal algorithm parameter $\rho$ from the portfolio for each instance.)
		\item The algorithm selector $\hat{f}$ returns high-performing parameter settings from the set $\hat{\cP}$ in the sense that \begin{equation}\frac{1}{N}\sum_{z \in \sample}u_{\hat{f}(z)}(z) \geq \frac{1}{N}\sum_{z \in \sample} \max_{\rho \in \hat{\cP}} u_{\rho}(z) - \epsilon.\label{eq:selector}\end{equation}
	\end{enumerate}
\end{definition}
For example, when algorithmic performance as a function of the parameters is piecewise constant, there are only a finite number of meaningfully different parameter values to choose among, one per piece. Then, since $\sum_{z \in \sample} \max_{\rho \in \hat{\cP}} u_{\rho}(z)$ is a submodular function of the portfolio $\hat{\cP}$, we can use a greedy algorithm to select the portfolio $\hat{\cP}$, and we obtain $\alpha = 1-\frac{1}{e}$ and $\beta = 0$, as we prove in Appendix~\ref{app:submodular}. Alternatively, an integer programming technique could be used to select the optimal portfolio from the finite set of candidate parameter values, in which case we would obtain $\alpha = 1$ and $\beta = 0$. Moreover, the value $\epsilon$ can be calculated directly from the training set.

The following theorem bounds the difference between the expected performance of the chosen selector $\hat{f}$ and an oracle that selects an optimal selector and an optimal portfolio. The full proof is in Appendix~\ref{app:procedure}.
\begin{restatable}{theorem}{overallBound}\label{thm:overall}
	Suppose that each dual function $u_z^*$ is piecewise constant with at most $t$ pieces.
	Given a training set $\sample \subseteq \cZ$ of size $N$, suppose we learn an $(\alpha, \beta, \epsilon)$-optimal portfolio $\hat{\cP} \subset \R$ and algorithm selector $\hat{f}: \cZ \to \hat{\cP}$ in $\cF$. With probability $1 - \delta$ over the draw of the training set $\sample \sim \dist^N$, \[\E_{z \sim \dist}\left[u_{\hat{f}(z)}(z)\right]
		\geq \alpha\max_{\cP : |\cP| \leq \kappa}\E\left[\max_{\rho \in \cP} u_{\rho}(z) \right] - \epsilon - \beta - \tilde O \left(H\sqrt{\frac{\bar{d} + \kappa}{N}}\right),\] where $\bar{d}$ is the Natarajan dimension of $\bar{\cF}$.
\end{restatable}
\begin{proof}[Proof sketch]
First, let $\cP^*$ be the optimal portfolio in the sense that \[\cP^* = \argmax_{\cP \subset \R : |\cP| \leq \kappa}\E_{z \sim \dist}\left[\max_{\rho \in \cP} u_{\rho}(z) \right].\] We use a Hoeffding bound to relate the expected performance of $\cP^*$ under the oracle algorithm selector and its average performance over the training set. We then use Definition~\ref{def:opt} to relate the latter quantity to the average performance of the learned selector $\hat{f}$ over the training set. Finally, we use Theorem~\ref{thm:nat} to relate the average performance of $\hat{f}$ to its expected performance. Putting all of these bounds together, we prove the theorem statement.
\end{proof}

By a parallel argument, we can obtain symmetric guarantees when our goal is to minimize rather than maximize a performance measure.

\section{Experiments}
We provide experiments that illustrate the tradeoff we investigated from a  theoretical perspective in the previous sections: as we increase the portfolio size, we can hope to include a well-suited parameter setting for any problem instance, but it becomes increasingly difficult to avoid overfitting.
We illustrate this in the context of integer programming algorithm configuration. We configure CPLEX, one of the most widely used commercial solvers. CPLEX uses the \emph{branch-and-cut (B\&C)} algorithm (branch-and-bound with cutting planes, primal heuristics, preprocessing, etc.) to solve integer programs (IPs). We tune an important parameter $\rho \in [0,1]$ of CPLEX that controls its \emph{variable selection policy}\footnote{We override the default variable selection of CPLEX 12.8.0.0 using the C API. All experiments were run on a
	64-core machine with 512 GB of RAM, a m4.16xlarge Amazon AWS instance, and a cluster of m4.xlarge Amazon AWS instances.} and has been studied extensively in prior research~\citep{Gauthier77:Experiments,Benichou71:Experiments,Beale79:Branch,Linderoth99:Computational,Achterberg09:SCIP, Balcan18:Learning}. We leave CPLEX's other techniques on and unchanged in order to compare against the state of the art. We provide a more detailed overview of CPLEX and the parameter we tune in Appendix~\ref{app:experiments}. At a high level, B\&C partitions the IP's feasible region, finding locally optimal solutions within the regions of the partition, and eventually verifies that the best solution found so far is globally optimal. It organizes this partition as a tree. 
As in prior research~\citep{Balcan18:Learning,Gupta20:Hybrid,Zarpellon20:Parameterizing}, our goal is to find parameter settings leading to small trees, so we define $u_{\rho}(z)$ to be the size of the tree B\&C builds. We aim to learn a portfolio $\hat{\cP}$ and  selector $\hat{f}$ resulting in small expected tree size $\E\left[u_{\hat{f}(z)}(z)\right]$.

\paragraph{Distribution over IPs.}
We analyze a distribution over IPs formulating the combinatorial auction winner determination problem under the OR-bidding language~\citep{Sandholm02:Algorithm}, which we generate using the Combinatorial
Auction Test Suite~\citep{Leyton00:Toward}.
We use the ``arbitrary'' generator with 200 bids and 100 goods, resulting in IPs with around 200 variables, and the ``regions'' generator with 400
bids and 200 goods, resulting in IPs with around 400 variables. We define a heterogeneous distribution $\dist$ as follows: with equal probability, we draw an instance from the ``arbitrary'' or ``regions'' distribution. To assign features to these IPs, we use all the features developed in prior research by~\citet{Leyton00:Toward} and \citet{Hutter14:Algorithm}, resulting in 140 features.

\paragraph{Experimental procedure.} We first learn a portfolio of size 10 in the following way. We draw a training set of $M = 1000$ IPs $z_1, \dots, z_M \sim \dist$ and solve for the dual functions $u_{z_1}^*, \dots, u_{z_M}^*$---which measure tree size as a function of the parameter $\rho$---using the algorithm described in Appendix D.1 of the paper by \citet{Balcan18:Learning}. These functions are piecewise-constant with at most $t$ pieces, for some $t \in \N$. Therefore, there are at most $Mt$ parameter settings leading to different algorithmic performance over the training set. Let $\bar{\cP}$ be this set of parameter settings. We use a greedy algorithm to select 10 parameter settings from $\bar{\cP}$. First, we find a parameter setting $\rho_1$ which minimizes average tree size over the training set: $\rho_1 \in \argmin \sum_{i = 1}^M u_{z_i^*}(\rho)$. Then, we find a parameter setting $\rho_2$ that minimizes average tree size when the better of $\rho_1$ or $\rho_2$ is used: $\rho_2 \in \argmin \sum_{i = 1}^M \min\left\{u_{z_i^*}(\rho),u_{z_i^*}\left(\rho_1\right)\right\}$. We continue greedily until we have a portfolio $\hat{\cP} = \left\{\rho_1, \dots, \rho_{10}\right\}$.

We then use a regression forest to select among parameter settings in the portfolio $\hat{\cP}$. Prior research~\citep{Hutter14:Algorithm} has illustrated that regression forests can be strong predictors of B\&C runtime. Here, we use them to predict B\&C tree size. A regression forest is a set $F = \left\{T_1, \dots, T_M\right\}$ of regression trees (which we reviewed in Section~\ref{sec:tree}). On an input IP $z$, the regression forest's prediction, denoted $h_{F}(z)$, is the average of the trees' predictions: $h_F(z) = \frac{1}{M} \sum_{i = 1}^M h_{T_i}(z)$.
We learn regression forests $F_1, \dots, F_{10}$ for each of the 10 parameter settings in the portfolio $\hat{\cP}$. We then define the algorithm selector $\hat{f}(z) = \rho_i$ where $i = \argmin \left\{h_{F_1}(z), \dots, h_{F_{10}}(z)\right\}$.

To learn the regression forest, we draw a training set $z_1, \dots, z_N \sim \dist$ of IPs (with $N$ specified below). For each parameter setting $\rho_i \in \hat{\cP}$ and IP $z_j$, we compute $u_{\rho_i}\left(z_j\right)$, the size of the tree B\&C builds using the parameter setting $\rho_i$. We then train the regression forest $F_i$ corresponding to the parameter setting $\rho_i$ using the labeled training set $\left\{\left(z_1, u_{z_1}\left(\rho_i\right)\right), \dots, \left(z_N, u_{z_N}\left(\rho_i\right)\right)\right\}$.
 We use Python's scikit-learn regression forest implementation~\citep{Pedregosa11:Scikit} with the default parameter settings.

  \savebox{\imagebox}{\includegraphics[width=0.3\textwidth]{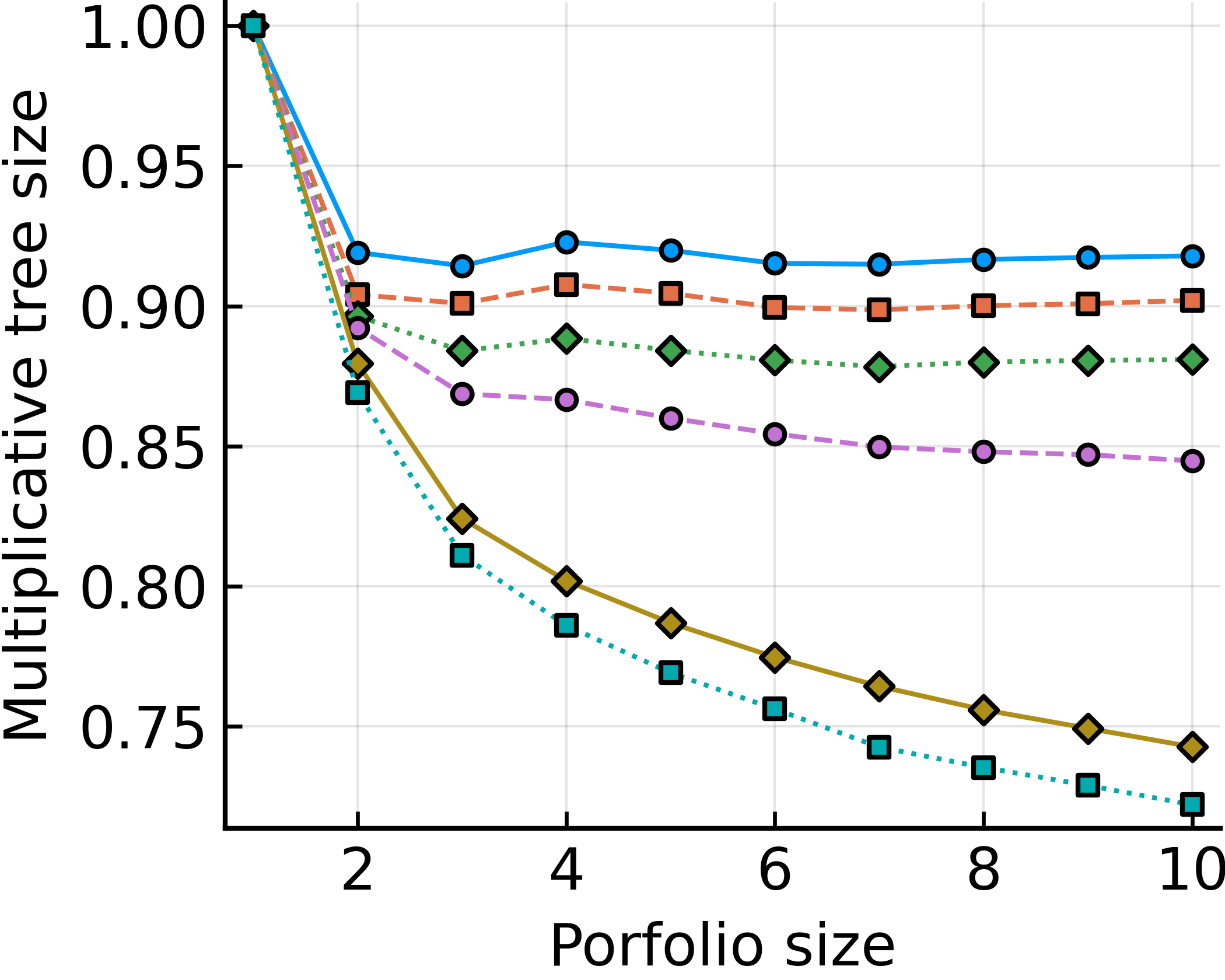}}%
\begin{figure}
	\centering
	\begin{subfigure}[t]{0.3\textwidth}
		\centering\usebox{\imagebox}
		\caption{Plot with portfolio sizes 1 through 10.}\label{fig:small}
	\end{subfigure}\qquad
	\begin{subfigure}[t]{0.28\textwidth}
		\centering\raisebox{\dimexpr.5\ht\imagebox-.5\height}{
			\includegraphics[width=\textwidth]{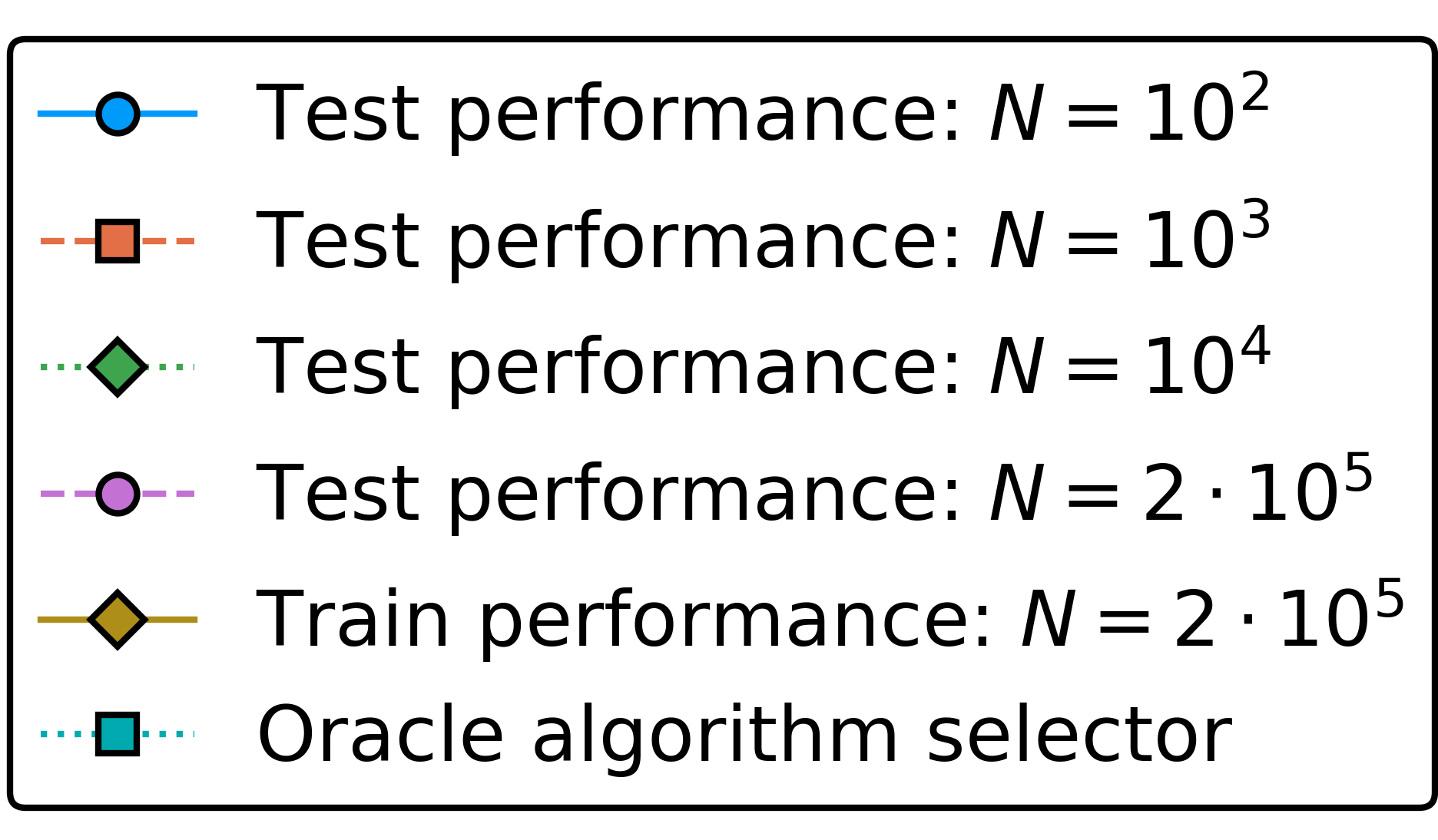}}
		\caption{Legend for Figures~\ref{fig:small} and \ref{fig:big}.}
	\end{subfigure}\qquad
	\begin{subfigure}[t]{0.3\textwidth}
		\includegraphics[width=\textwidth]{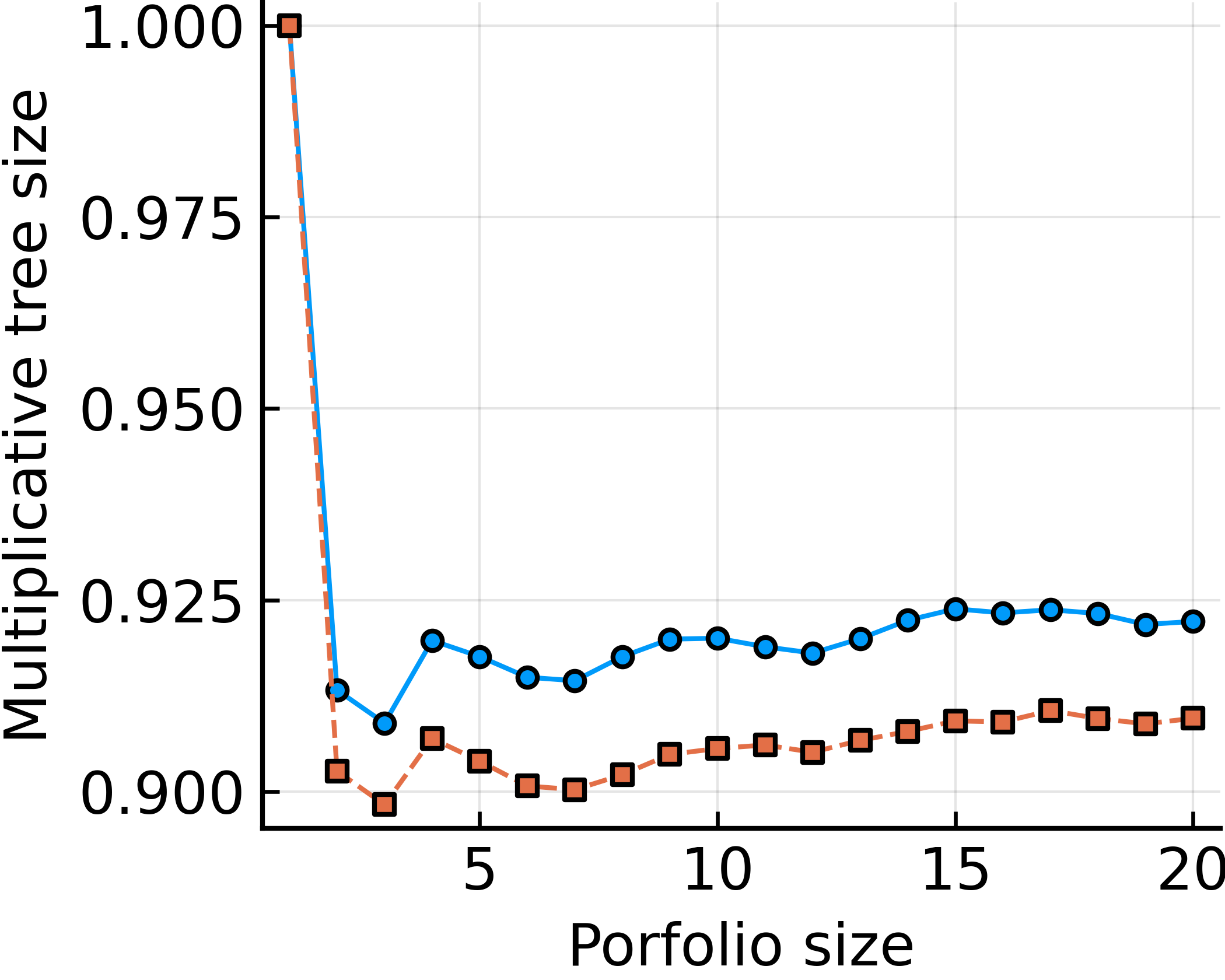}
		\caption{Plot with portfolio sizes 1 through 20.}\label{fig:big}
	\end{subfigure}
	\caption{In Figures~\ref{fig:small} and \ref{fig:big}, we plot the multiplicative tree size improvement we obtain as we increase both the portfolio size along the horizontal axis and the size of the training set, denoted $N$. Fixing a training set size and letting $\hat{v}_{\kappa}$ be the average tree size we obtain over the test set using a portfolio of size $\kappa$ (see Equation~\eqref{eq:test}), we plot $\hat{v}_{\kappa}/\hat{v}_1$. In Figure~\ref{fig:small}, the portfolio size ranges from 1 to 10 and the training set size $N$ ranges from 100 to 200,000. In Figure~\ref{fig:big}, the portfolio size ranges from 1 to 20 and the training set size ranges from 100 to 1000. In Figure~\ref{fig:small}, we also plot a similar curve for the test performance of the oracle algorithm selector, as well as the training performance of the learned algorithm selector when $N = 2 \cdot 10^5$.}\label{fig:train_v_test}
\end{figure}

In Figure~\ref{fig:small}, we plot the performance of the regression forests as we increase the sizes of both the training set and the portfolio. We denote the training set size as $N$, which ranges from 100 to 200,000.
 For a given choice of   $N$, we first train the 10 regression forests $F_1, \dots, F_{10}$ using the method described above. We then evaluate performance as a function of the portfolio size. Specifically, for each portfolio size $\kappa \in [10]$, we define an algorithm selector $\hat{f}_{\kappa}(z) = \rho_i$ where $i = \argmin \left\{h_{F_1}(z), \dots, h_{F_{\kappa}}(z)\right\}$. We draw $N_t = 10^4$ test instances $\sample_t \sim \dist^{N_t}$ and evaluate the performance of  $\hat{f}_{\kappa}$ on the test set. We denote the average test performance as \begin{equation}\hat{v}_{\kappa} = \frac{1}{N_t} \sum_{z \in \sample_t} u_{\hat{f}_{\kappa}\left(z\right)}\left(z\right).\label{eq:test}\end{equation} In Figure~\ref{fig:small}, we plot the multiplicative performance improvement we obtain as we increase $\kappa$. Specifically, we plot $\hat{v}_{\kappa} / \hat{v}_1$. These are the blue solid $(N = 10^2)$, orange dashed $(N = 10^3)$, green dotted $(N = 10^4)$, and purple dashed $(N = 2 \cdot 10^5)$ lines.  By the iterative fashion we constructed the portfolio, $\hat{v}_1$ is the performance of the best single parameter setting for the particular distribution, so $\hat{v}_1$ is already highly optimized.

We plot a similar curve for the test performance of the oracle algorithm selector which always selects the optimal parameter setting from the portfolio. Specifically, for each portfolio size $\kappa \in [10]$, let $f^*_{\kappa}$ be the oracle algorithm selector $f_{\kappa}^*(z) = \argmin_{\rho_1, \dots, \rho_{\kappa}} u_{\rho_i}(z)$. Given a test set $\sample_t \sim \dist^{N_t}$, we define the average test performance of $f^*_{\kappa}$ as \[v_{\kappa}^* = \frac{1}{N_t} \sum_{z \in \sample_t} u_{f^*_{\kappa}\left(z\right)}\left(z\right).\] The blue dotted line equals $v_{\kappa}^*/v_1^*$ as a function of $\kappa$.

Finally, when the training set is of size $N = 2 \cdot 10^5$, we provide a similar curve for the training performance of the learned algorithm selectors $\hat{f}_{\kappa}$. Letting $z_1, \dots, z_N$ be the training set, we denote the average training performance as \[\tilde{v}_{\kappa} = \frac{1}{N} \sum_{i = 1}^{N} u_{f^*_{\kappa}\left(z_i\right)}\left(z_i\right).\]  The yellow solid line equals $\tilde{v}_{\kappa} / \tilde{v}_1$ as a function of the portfolio size $\kappa$.

In Figure~\ref{fig:big}, we plot $\hat{v}_{\kappa} / \hat{v}_1$ as a function of the portfolio size $\kappa$ for larger portfolio sizes ranging from 1 to 20. We greedily extend the portfolio $\hat{\cP}$ to include an additional 20 parameter settings. We then train 20 regression forests using freshly drawn training sets of size 100 and 1000. This plot illustrates the fact that as we increase the portfolio size, overfitting causes test performance to worsen.

\paragraph{Discussion.} Focusing first on test performance using the largest training set size $N = 2\cdot 10^5$, we see that test performance continues to improve as we increase the portfolio size, though training and test performance steadily diverge. This illustrates the tradeoff we investigated from a theoretical perspective in this paper: as we increase the portfolio size, we can hope to include a well-suited parameter setting for every instance, but the generalization error will worsen. Figure~\ref{fig:big} shows that for a given training set size, there is a portfolio size after which test performance actually starts to get strictly worse, as our theory predicts. In other words, we observe overfitting: the learned algorithm selector has strong average performance over the training set but poor test performance.

\section{Conclusions}
We provided guarantees for learning a portfolio of parameter settings in conjunction with an algorithm selector for that portfolio. We provided a tight (up to log factors) bound on the number of samples sufficient and necessary to ensure that the selector's average performance on the training set generalizes to its expected performance on the real unknown problem instance distribution. Our guarantees apply in the widely-applicable setting where the algorithm's performance on any input problem instance is a piecewise-constant function of its parameters. Our theoretical bounds indicate that even with an extremely simple algorithm selector, we cannot hope to avoid overfitting in the worst-case if the portfolio is large. Thus, there is a tradeoff when increasing the portfolio size, since a large portfolio allows for the possibility of including a strong parameter setting for every instance, but this potential for performance improvement is overshadowed by a worsening propensity towards overfitting. We concluded with experiments illustrating this tradeoff in the context of integer programming.
A direction for future research is to understand how the diversity of a portfolio impacts its generalization error, since algorithm portfolios are often expressly designed to be diverse.

\subsection*{Acknowledgments}
This material is based on work supported by the National Science Foundation under grants CCF-1535967, CCF-1733556, CCF-1910321, IIS-1617590, IIS-1618714, IIS-1718457, IIS-1901403, and SES-1919453; the ARO under awards W911NF1710082 and W911NF2010081; the Defense Advanced Research Projects Agency under cooperative agreement HR00112020003; an AWS Machine Learning Research Award; an Amazon Research Award; a Bloomberg Research Grant; a Microsoft Research Faculty Fellowship; an IBM PhD fellowship; and a fellowship from Carnegie Mellon University’s Center for Machine Learning and Health.

\bibliographystyle{plainnat}
\bibliography{../dairefs}
\appendix
\section{Sample complexity proofs from Section~\ref{sec:sample}}\label{app:sample}
\nat*

\begin{proof}
Let $z_1, \dots, z_N \in \cZ$ be a set of problem instances that is shattered by $\cU_{\cF}$, as witnessed by the points $t_1, \dots, t_N \in \R$.
By definition, this means that \begin{align}2^N = \left|\left\{\begin{pmatrix}\textbf{1}_{\left\{u_{f(z_1)} (z_1)\leq t_1\right\}}\\
				\vdots\\
				\textbf{1}_{\left\{u_{f(z_N)}(z_N) \leq t_N\right\}}\end{pmatrix} : f \in \cF\right\}\right|\leq \left|\left\{\begin{pmatrix}u_{f(z_1)} (z_1)\\
				\vdots\\
				u_{f(z_N)}(z_N) \end{pmatrix} : f \in \cF\right\}\right|.\label{eq:shatter}\end{align}
Since each dual function $u_{z_i}^*$ is piecewise-constant with at most $t$ pieces, we know there are $M \leq Nt$ intervals $I_1, \dots, I_M$ partitioning $\R$ where for any interval $I_j$ and any problem instance $z_i$, $u_{z_i}^*(\rho)$ is constant across all $\rho \in I_j$.
We assume that the intervals are ordered so that if $j < j'$, then the points in $I_j$ are smaller than the points in $I_{j'}$

Let $J= \left(j_1, \dots, j_{\bar{\kappa}}\right) \in [M]^{\bar{\kappa}}$ be a vector of $\bar{\kappa} \leq \kappa$ interval indices with $j_1 \leq j_2 \leq \cdots \leq j_{\bar{\kappa}}$. Let $\cF_{J} \subseteq \cF$ be the set of functions $f \in \cF$ with the following property: letting $\rho_1 < \rho_2 < \cdots < \rho_{\bar{\kappa}}$ be the parameter settings $f$ maps to (i.e., $\{f(z) : z \in \cZ\} = \left\{\rho_1, \dots \rho_{\bar{\kappa}}\right\}$), we have that the $i^{th}$ parameter setting is in the $i^{th}$ interval: $\rho_1 \in I_{j_1}, \dots, \rho_{\bar{\kappa}} \in I_{j_{\bar{\kappa}}}$.
Since $I_1, \dots, I_M$ partition $\R$ and since each function $f \in \cF$ maps to at most $\kappa$ parameter settings, $\cF= \cup_{J} \cF_J$. Together with Equation~\eqref{eq:shatter}, this means that \begin{equation}2^N \leq \sum_{\bar{\kappa} = 1}^{\kappa}\sum_{J \in [M]^{\bar{\kappa}}} \left|\left\{\begin{pmatrix}u_{f(z_1)} (z_1)\\
				\vdots\\
				u_{f(z_N)}(z_N) \end{pmatrix} : f \in \cF_{J}\right\}\right|.\label{eq:union}\end{equation}

Fix a particular set $J= \left(j_1, \dots, j_{\bar{\kappa}}\right) \in [M]^{\bar{\kappa}}$ as defined above.
 For each algorithm selector $f \in \cF_J$, let $f_0: \cZ \to J$ be a function that indicates which of the $\bar{\kappa}$ intervals $I_{j_1}, \dots, I_{j_{\bar{\kappa}}}$ the parameter setting $f(z)$ falls in. In other words, $f_0(z) = j$ if and only if $f(z) \in I_j$.
Recall that for any $i \in [N]$ and $j \in J$,  $u_{f(z_i)}(z_i)$ is constant across all $f \in \cF_{J}$ with $f(z_i) \in I_j$. Therefore, even if we only know which of the $\bar{\kappa}$ intervals $f\left(z_i\right)$ falls in and not the function $f$ itself, we can correctly infer the value $u_{f\left(z_i\right)}\left(z_i\right)$. Said another way, if we only know the value $f_0\left(z_i\right) \in [\bar{\kappa}]$, we can infer the value $u_{f\left(z_i\right)}\left(z_i\right)$. Aggregating this logic across all $N$ problem instances, given a vector $\left(f_0\left(z_1\right), \dots, f_0(z_N)\right)$ we can directly infer the vector $\left(u_{f(z_1)} (z_1), \dots, u_{f(z_N)} (z_N)\right)$. This implies that \begin{equation}\left|\left\{\begin{pmatrix}u_{f(z_1)} (z_1)\\
	\vdots\\
	u_{f(z_N)}(z_N) \end{pmatrix} : f \in \cF_J\right\}\right|
\leq \left|\left\{\begin{pmatrix}f_0(z_1)\\
	\vdots\\
	f_0(z_N) \end{pmatrix} : f \in \cF_J\right\}\right|.\label{eq:uf_f_0}\end{equation}

Next, we use a similar logic to show that
 \begin{equation}\left|\left\{\begin{pmatrix}f_0(z_1)\\
				\vdots\\
				f_0(z_N) \end{pmatrix} : f \in \cF_J\right\}\right|\\
			\leq \left|\left\{\begin{pmatrix}\bar{f}(z_1)\\
				\vdots\\
				\bar{f}(z_N) \end{pmatrix} : f \in \cF_J\right\}\right|.\label{eq:f_0_bar}\end{equation}
			To see why, suppose we only know the value $\bar{f}(z_i)$ and not the function $f$ itself. For ease of notation, say $\ell = \bar{f}(z_i)$. By definition of $\bar{f}$, we know that $f(z_i)$ is the $\ell^{th}$-smallest parameter setting that the function $f$ maps to. By definition of the function $f_0$, this implies that $f_0\left(z_i\right) = j_{\ell}$. Therefore, if we only know the value $\bar{f}(z_i)$ and not the function $f$ itself, we can correctly infer the value $f_0\left(z_i\right)$. Again, aggregating this logic across all $N$ problem instances, given a vector $\left(\bar{f}\left(z_1\right), \dots, \bar{f}(z_N)\right)$ we can directly infer the vector $\left(f_0\left(z_1\right), \dots, f_0(z_N)\right)$. This implies that Equation~\eqref{eq:f_0_bar} holds.
			
Combining Equations~\ref{eq:uf_f_0} and \eqref{eq:f_0_bar} with Natarajan's lemma~\citep{Natarajan89:Learning}, we have that \[\left|\left\{\begin{pmatrix}u_{f(z_1)} (z_1)\\
			\vdots\\
			u_{f(z_N)}(z_N) \end{pmatrix} : f \in \cF_J\right\}\right|\leq N^{\bar{d}}\bar{\kappa}^{2\bar{d}}.\]
Combining this fact, the fact that $M \leq Nt$, and Equation~\eqref{eq:union}, we have that $2^N \leq \kappa(Nt)^{\kappa} N^{\bar{d}}\kappa^{2\bar{d}}$, which implies that $N = O\left(\left(\kappa + \bar{d}\right)\log \left(\kappa + \bar{d}\right) + \kappa\log t\right)$.
\end{proof}

\lb*
\begin{proof}
	Let $\cZ = (0,1]$. For each parameter setting $\rho \in \R$, define $u_{\rho}(z) = \textbf{1}_{\{z \leq \rho\}}$. As claimed, each dual function $u_z^* : \R \to \R$ is piecewise-constant with 1 discontinuity. In this case, the function in $\cU_{\cF}$ map $\cZ$ to $\{0,1\}$.
	In the special case where the range of the function class is $\{0,1\}$, pseudo-dimension is typically referred to as \emph{VC dimension}, which we denote as $\VC\left(\cU_{\cF}\right) .$

	Let $\kappa, \bar{d} \geq 2$ be two arbitrary integers. We split this proof into two cases: $\bar{d} \geq \kappa$ and $\kappa > \bar{d}$. In both cases, we exhibit a class of  selectors $\cF$ that satisfies the properties in the theorem statement and we prove that $\VC\left(\cU_{\cF}\right) \geq \max\left\{\kappa, \bar{d}\right\} = \Omega\left(\kappa + \bar{d}\right)$.
	
	\begin{claim}\label{claim:d}
		Suppose $\bar{d} \geq \kappa$. There exists a class of  selectors $\cF$ that satisfies the properties in the theorem statement and $\VC\left(\cU_{\cF}\right) = \bar{d}$.
	\end{claim}
	\begin{proof}[Proof of Claim~\ref{claim:d}]
		Let $\cF \subseteq \{0,1\}^{\cZ}$ be any set of binary functions with VC dimension $\bar{d}$. As required, each  selector $f \in \cF$ maps to at most $\kappa$ parameter settings $(|\{f(z) : z \in \cZ\}| \leq 2 \leq \kappa)$. Moreover, $\bar{\cF} = \cF$, so the Natarajan dimension of $\bar{\cF}$ equals the VC dimension of $\cF$, which is $\bar{d}$.
		
		For any instance $z \in \cZ$ and function $f \in \cF$, \[u_{f(z)}(z) = \begin{cases} 1 &\text{if } z \leq f(z)\\
			0 &\text{if } z > f(z).\end{cases}\] Since $z \in (0,1]$ and $f(z) \in \{0,1\}$, this implies that $u_{f(z)}(z) = f(z)$. Therefore, $\VC\left(\cU_{\cF}\right) = \VC(\cF) = \bar{d}$.
	\end{proof}
	
	\begin{claim}\label{claim:kappa}
		Suppose $\kappa > \bar{d}$. There exists a class of  selectors $\cF$ that satisfies the properties in the theorem statement and $\VC\left(\cU_{\cF}\right) \geq \kappa$.
	\end{claim}
	\begin{proof}[Proof of Claim~\ref{claim:kappa}]
		We begin by partitioning $\cZ = (0,1]$ into $\kappa$ intervals $Z_1, \dots, Z_{\kappa}$, where $Z_i = \left(\frac{i-1}{\kappa}, \frac{i}{\kappa}\right]$.
		For each set $T \subseteq [\kappa]$, define an  selector $f_{T} : \cZ \to \R$ as follows. For any $z \in \cZ = (0,1]$, let $i \in [\kappa]$ be the index of the interval $z$ lies in, i.e., $z \in Z_i$. We define \[f_{T}(z)= \begin{cases} \frac{i}{\kappa} &\text{if } i \in T\\
			\frac{i}{\kappa} - \frac{1}{2\kappa}&\text{if } i \not\in T.\end{cases}\]
		Let $\cF= \left\{f_{T} : T \subseteq [\kappa]\right\}$. For every function $f \in \cF$, $\bar{f}(z)$ equals the index $i \in [\kappa]$ such that $z \in Z_i$. Therefore, $\left|\bar{\cF}\right| = 1$, so the Natarajan dimension of $\bar{\cF}$ is $0 < \bar{d}$.
		
		Define $\sample = \left\{\frac{1}{\kappa}, \frac{2}{\kappa}, \dots, \frac{\kappa - 1}{\kappa}, 1\right\} \subset \cZ$. We prove that $\sample$ is shattered by $\cU_{\cF}$. Let $T \subseteq [\kappa]$ be an arbitrary subset. If $i \in T$, then $f_T\left(\frac{i}{\kappa}\right) = \frac{i}{\kappa}$, so \[u_{f_T\left(\frac{i}{\kappa}\right) }\left(\frac{i}{\kappa}\right)  = u_{\frac{i}{\kappa}}\left(\frac{i}{\kappa}\right) = \textbf{1}_{\left\{\frac{i}{\kappa}\leq\frac{i}{\kappa}\right\}} = 1.\] If $i \not\in T$, then $f_T\left(\frac{i}{\kappa}\right) = \frac{i}{\kappa} - \frac{1}{2\kappa}$, so \[u_{f_T\left(\frac{i}{\kappa}\right) }\left(\frac{i}{\kappa}\right)  = u_{\frac{i}{\kappa} - \frac{1}{2\kappa}}\left(\frac{i}{\kappa}\right) = \textbf{1}_{\left\{\frac{i}{\kappa}\leq\frac{i}{\kappa} - \frac{1}{2\kappa}\right\}} = 0.\] Therefore, $\sample$ is shattered by $\cU_{\cF}$, so the VC dimension of $\cU_{\cF}$ is at least $\kappa$.
	\end{proof}
	These two claims illustrate that $\VC\left(\cU_{\cF}\right) \geq \max\left\{\kappa, \bar{d}\right\} = \Omega\left(\kappa + \bar{d}\right)$.
\end{proof}

\subsection{Regression tree performance models}\label{app:rt}
\rt*

\begin{proof}
	In this proof, to simplify notation, we will denote the feature vector $\phi(z)$ as $\vec{z} \in \R^m$. For each $\kappa$-tuple of regression trees $\vec{T} = \left(T_1, \dots, T_{\kappa}\right)$, let $g_{\vec{T}} : \cZ \to [\kappa]$ be a function where $g_{\vec{T}}(z) = \argmax_{i \in [\kappa]} \left\{h_{T_i}(z)\right\}$. By definition, the set $\bar{\cF}_R$ consists of the functions $g_{\vec{T}}$ across all $\kappa$-tuples of regression trees $\vec{T}$ with at most $\ell$ leaves. We will assume, without loss of generality, that all trees are full.
	
	Let $N$ be the the Natarajan dimension of $\bar{\cF}_R$ and let $z_1, \dots, z_N \in \cZ$ be a set of $N$ problem instances that are multi-class shattered by $\bar{\cF}_R$. 
	This implies that \begin{equation}2^N \leq \left|\left\{\begin{pmatrix}g_{\vec{T}}(z_1)\\
		\vdots\\
		g_{\vec{T}}(z_N)\end{pmatrix} : \vec{T} \text{ is a }\kappa\text{-tuple of regression trees}\right\}\right|.\label{eq:rt_partition}\end{equation}
	In this proof, we show that the right-hand-side of this inequality is bounded by $m^{\kappa(\ell - 1)}(N\ell)^{\ell\kappa}(\kappa\ell)^{\kappa\ell}$, which implies that $N = O(\ell\kappa \log (\ell \kappa m))$.
	
	To this end, we begin by focusing on a single regression tree $T$.
	 We analyze the number of ways the tree can partition the instances $\vec{z}_1, \dots, \vec{z}_N$ as we vary the parameters of $T$.
	 Each internal node of $T$ performances an inequality test on some feature of the input, so it is defined by a feature index $i \in [m]$ and a threshold $\theta \in \R$. Since there are $\ell$ leaves, there are $\ell-1$ internal nodes.
	First, we fix the indices of all internal nodes, which leaves $\ell - 1$ real-valued thresholds to analyze.
	At a particular internal node $\nu$, let $j$ be the index of the feature on which the node performs an inequality test and let $\theta_{\nu}$ be the threshold (where the index $j$ is fixed by the threshold $\theta_{\nu}$ is not fixed). Whether or not the instance $\vec{z}_i$ would be sorted into the left or right child of the node depends on whether or not \begin{equation}z_i[j] \leq \theta_{\nu}\label{eq:regression_hyperplane}\end{equation} (where $z_i[j]$ is the $j^{th}$ coordinate of the vector $\vec{z}_i$).
 	For each problem instance $\vec{z}_i$, there are therefore $\ell - 1$ hyperplanes splitting the set of thresholds $\R^{\ell - 1}$ into regions where if we use thresholds from within any one region, the path that the instance $\vec{z}_i$ takes through the tree (from root to leaf) is constant. The same holds for all $N$ problem instances, leading to a total of $N(\ell-1)$ hyperplanes in $\R^{\ell - 1}$. In total, these hyperplanes split $\R^{\ell - 1}$ into at most $(N\ell)^{\ell}$ regions where if we use the thresholds from within any one region, the path that each of the $N$ problem instances takes through the tree is constant~\citep{Buck43:Partition}. Since this is true no matter how we fix the feature indices of each interval node, tuning all parameters of the tree $T$ (both the feature indices and the thresholds) can induce at most $m^{\ell - 1}(N\ell)^{\ell}$ different partitions of the $N$ problem instances.
 	
 	Said another way, for any tree $T$ and instance $z \in \cZ$, let $\lambda_T(z) \in [\ell]$ be the index of the leaf that the instance $z$ is mapped to as we apply the inequality tests defined by the internal nodes of $T$. As we vary the tree $T$, the vector $\left(\lambda_T\left(z_1\right), \dots, \lambda_T\left(z_N\right)\right) \in [\ell]^N$ will take on at most $m^{\ell - 1}(N\ell)^{\ell}$ different values.
 	
We now aggregate this reasoning across all $\kappa$ regression trees. For any instance $z \in \cZ$ and any $\kappa$-tuple of regression trees $\vec{T} = \left(T_1, \dots, T_{\kappa}\right)$, let $\Lambda_{\vec{T}}(z) \in [\ell]^{\kappa}$ be a vector where for each $j \in [\kappa]$, the $j^{th}$ component of $\Lambda_{\vec{T}}(z)$ is the index of the leaf that the instance $z$ is mapped to as we apply the inequality tests defined by the tree $T_j$. In other words, $\Lambda_{\vec{T}}(z) = \left(\lambda_{T_1}(z), \dots, \lambda_{T_{\kappa}}(z)\right)$. As we vary $\vec{T}$, the matrix \[\begin{pmatrix}
\Lambda_{\vec{T}}\left(z_1\right) & \hdots &
\Lambda_{\vec{T}}\left(z_N\right)
\end{pmatrix} = \begin{pmatrix} \lambda_{T_1}\left(z_1\right) &\cdots &\lambda_{T_{1}}\left(z_N\right) \\
	\vdots & \ddots & \vdots\\
	\lambda_{T_{\kappa}}\left(z_1\right) &\cdots &\lambda_{T_{\kappa}}\left(z_N\right) 
	\end{pmatrix}\] will take on at most $m^{(\ell - 1)\kappa}(N\ell)^{\ell\kappa}$ different values. After all, the first row of the matrix can take on at most $m^{\ell - 1}(N\ell)^{\ell}$ different values as we vary the tree $T_1$, the second row can take on at most $m^{\ell - 1}(N\ell)^{\ell}$ different values as we vary $T_2$, and so on.
 	
Now, consider the set of all $\kappa$-tuples of regression trees $\vec{T}$ where the matrix $\left(\Lambda_{\vec{T}}\left(z_1\right), \dots,
\Lambda_{\vec{T}}\left(z_N\right)\right)$ is constant. Across all such $\vec{T}= \left(T_1, \dots, T_{\kappa}\right)$, we know exactly which leaf each instance $z_i$ maps to for all $\kappa$ trees.
For each each instance $z_i$, the tree with the largest label---or in other words, the value of the multi-class function $g_{\vec{T}}(z_i)$---only depends on the relative order of the leaves' predictions. Since there is a total of $\kappa\ell$ leaves, there are at most $(\kappa\ell)^{\kappa\ell}$ such orderings. Combining this bound with the bound from the previous paragraph, we have that \[\left|\left\{\begin{pmatrix}g_{\vec{T}}(z_1)\\
	\vdots\\
	g_{\vec{T}}(z_N)\end{pmatrix} : \vec{T} \text{ is a }\kappa\text{-tuple of regression trees}\right\}\right| \leq m^{(\ell - 1)\kappa}(N\ell)^{\ell\kappa}(\kappa\ell)^{\kappa\ell}.\]
From Equation~\eqref{eq:rt_partition}, we have that $2^N \leq m^{(\ell - 1)\kappa}(N\ell)^{\ell\kappa}(\kappa\ell)^{\kappa\ell}$, so $N = O(\ell\kappa \log (\ell \kappa m))$.
\end{proof}

\subsection{Clustering-based algorithm selectors}\label{app:cluster}
\cluster*
\begin{proof}In this proof, to simplify notation, we will denote the feature vector $\phi(z)$ as $\vec{z} \in \R^m$.
	For each matrix $X\in \R^{m \times \kappa}$, let $g_X : \cZ \to [\kappa]$ be a function where \[g_X(z) = \argmin_{i \in [\kappa]}\left\{\norm{\vec{x}_i - \vec{z}}^p_p\right\}.\]
	By definition, $\bar{\cF}_C= \left\{g_X : X \in \R^{m \times \kappa}\right\}$.
	Let $N$ be the Natarajan dimension of $\bar{\cF}_C$ and let $z_1, \dots, z_N \in \cZ$ be a set of $N$ problem instances that are multi-class shattered by $\bar{\cF}_C$. 
	This implies that \begin{equation}2^N \leq \left|\left\{\begin{pmatrix}g_X(z_1)\\
		\vdots\\
		g_X(z_N)\end{pmatrix} : X \in \R^{m \times \kappa}\right\}\right|.\label{eq:partition_clus}\end{equation}
	In this proof, we analyze the partition of the parameter space $\R^{m \times \kappa}$ into regions where in any one region $R \subseteq \R^{m \times \kappa}$, across all matrices $X \in R$, the vector $\left(g_X(z_1), \dots,
	g_X(z_N)\right)$ is constant.
	
	We begin by subdividing $\R^{m \times \kappa}$ into regions $P_1, \dots, P_T\subseteq \R^{m \times \kappa}$ where in any one region $P$, across all $X \in P$, either the $\ell^{th}$ component of $\vec{z}_q$ is smaller than the $\ell^{th}$ component of $\vec{x}_j$, i.e. $z_q[\ell] \leq x_j[\ell]$, or vice versa (but not both) for all $q \in [N],$ $j \in [\kappa]$, and $\ell \in [m]$. This is partition is defined by $N \kappa m$ hyperplanes in $\R^{m\kappa}$, so there are $T \leq (N\kappa m + 1)^{m \kappa}$ such regions~\citep{Buck43:Partition}.
	
	Next, fix one of these $T$ regions $P \subseteq \R^{m \times \kappa}$. Without loss of generality, assume that the $\ell^{th}$ component of $\vec{z}_q$ is smaller than the $\ell^{th}$ component of $\vec{x}_j$, i.e. $z_q[\ell] \leq x_j[\ell]$ for all $q \in [N],$ $j \in [\kappa]$, and $\ell \in [m]$.
	For any two labels $i,j \in [\kappa]$ and any $q \in [N]$, whether or not \begin{equation}\norm{\vec{x}_i - \vec{z}_q}^p_p \geq \norm{\vec{x}_j - \vec{z}_q}^p_p\label{eq:norm}\end{equation} directly depends on the sign of the polynomial \[h_{q,i,j}(X):= \sum_{\ell = 1}^m \left(x_i[\ell] - z_q[\ell]\right)^p - \left(x_j[\ell] - z_q[\ell]\right)^p.\]
	We know there are at most $\left(N\kappa^2p\right)^{m\kappa}$ regions partitioning $P$ so that in any one region $R$, across all $X \in R$, either $h_{q,i,j}(X) \leq 0$ or $h_{q,i,j}(X) > 0$ (but not both) for all $q \in [N]$ and $i,j \in [\kappa]$~\citep{Anthony09:Neural}.
	For any such region $R$, across all $X \in R$, all pairwise comparisons as in Equation~\eqref{eq:norm} are fixed, so the vector $\left(g_X(z_1), \dots,
	g_X(z_N)\right)$ is constant.
	In total, there are at most $(N\kappa m + 1)^{m \kappa}\left(N\kappa^2p\right)^{m\kappa} \leq (2N^2\kappa^3 p \ell)^{m \kappa}$ regions, which implies that \[\left|\left\{\begin{pmatrix}g_X(z_1)\\
		\vdots\\
		g_X(z_N)\end{pmatrix} : X \in \R^{m \times \kappa}\right\}\right| \leq (2N^2\kappa^3 p \ell)^{m \kappa}.\]
Combining this inequality with Equation~\eqref{eq:partition_clus}, we have that $2^N \leq (2N^2\kappa^3 p \ell)^{m \kappa}$, so $N = O(mk \log (m\kappa p))$.
\end{proof}

\section{Proof of Theorem~\ref{thm:overall}}\label{app:procedure}

\overallBound*
\begin{proof}
	First, let \[\cP^* = \argmax_{\cP \subset \R : |\cP| \leq \kappa}\E_{z \sim \dist}\left[\max_{\rho \in \cP} u_{\rho}(z) \right].\]
	A Hoeffding bound implies that with probability $1-\delta$, \[\E_{z \sim \dist}\left[\max_{\rho \in \cP^*} u_{\rho}(z) \right] \leq \frac{1}{N}\sum_{z \in \sample} \max_{\rho \in \cP^*} u_{\rho}(z) + \tilde O\left(H\sqrt{\frac{1}{N}}\right).\]
	Combining this inequality with Definition~\ref{def:opt}, we have that \begin{align*}\E_{z \sim \dist}\left[\max_{\rho \in \cP^*} u_{\rho}(z) \right] &\leq \frac{1}{\alpha}\left(\frac{1}{ N}\sum_{z \in \sample} \max_{\rho \in \hat{\cP}} u_{\rho}(z) + \beta\right)+  O\left(H\sqrt{\frac{1}{N}\log \frac{1}{\delta}}\right)\\
	&\leq \frac{1}{\alpha}\left(\frac{1}{ N}\sum_{z \in \sample} u_{\hat{f}(z)}(z) + \epsilon + \beta\right)+ O\left(H\sqrt{\frac{1}{N} \log \frac{1}{\delta}}\right).\end{align*}
	From Theorem~\ref{thm:nat}, we know that with probability $1-\delta$, \[\E_{z \sim \dist}\left[\max_{\rho \in \cP^*} u_{\rho}(z) \right] \leq \frac{1}{\alpha}\left(\E_{z \sim \dist} \left[u_{\hat{f}(z)}(z)\right] + \tilde O \left(H\sqrt{\frac{\bar{d} + \kappa}{N}}\right) + \epsilon + \beta\right).\] Therefore, the theorem statement holds.
\end{proof}

\section{Connection to submodularity}\label{app:submodular}
Since each dual function $u_z^*(\rho)$ is piecewise-constant with at most $t$ pieces, on any training set $\sample = \left\{z_1, \dots, z_N\right\} \subseteq \cZ$, there are at most $Nt$ parameter settings leading to different algorithmic performance over this training set. In other words, \[\left|\left\{\begin{pmatrix} u_{z_1}^*(\rho)\\
		\vdots\\
		u_{z_N}^*(\rho)\end{pmatrix} : \rho \in \R\right\}\right| \leq Nt.\]
Let $\bar{\cP} \subseteq \R$ be a set of at most $Nt$ parameters such that \[\left\{\begin{pmatrix} u_{z_1}^*(\rho)\\
		\vdots\\
		u_{z_N}^*(\rho)\end{pmatrix} : \rho \in \R\right\} = \left\{\begin{pmatrix} u_{z_1}^*(\rho)\\
		\vdots\\
		u_{z_N}^*(\rho)\end{pmatrix} : \rho \in \bar{\cP}\right\}.\]
For any $T \subseteq \bar{\cP}$, let \[U(T) = \sum_{i = 1}^N \max_{\rho \in T} u_{\rho}(z_i).\]

\begin{theorem}\label{thm:submodular}
	The function $U$ is monotone and submodular.
\end{theorem}
\begin{proof}
For any $z_i$, let $U_i : 2^{\cP^*} \to \R$ be the function $U_i(T) = \max_{\rho \in T} u_{\rho}(z_i).$
We will prove that each function $U_i$ is submodular. The theorem then follows because the class of submodular functions is closed under non-negative linear combinations.
To this end, let $T \subseteq \cP^*$ be an arbitrary subset of $\cP^*$ and let $\rho_1, \rho_2 \in \cP^* \setminus T$ be any two parameter settings in $\cP^*$ but not in $T$. We want to prove that \begin{equation}\max_{\rho \in T \cup \{\rho_1\}} u_{\rho}(z_i) + \max_{\rho \in T \cup \{\rho_2\}} u_{\rho}(z_i) \geq \max_{\rho \in T \cup \{\rho_1, \rho_2\}} u_{\rho}(z_i) + \max_{\rho \in T} u_{\rho}(z_i).\label{eq:submodular}\end{equation}
Without loss of generality, suppose that $u_{\rho_1}(z_i) \geq u_{\rho_2}(z_i)$.
Let $\bar{\rho} \in \argmax_{\rho \in T} u_{\rho}(z_i)$.
There are three cases:
		\begin{itemize}
			\item In the first case, $u_{\bar{\rho}}(z) \geq u_{\rho_1}(z_i) \geq u_{\rho_2}(z_i)$, so \[
			\max_{\rho \in T \cup \{\rho_1\}} u_{\rho}(z_i) + \max_{\rho \in T \cup \{\rho_2\}} u_{\rho}(z_i) = 2u_{\bar{\rho}}(z_i) = \max_{\rho \in T \cup \{\rho_1, \rho_2\}} u_{\rho}(z_i) + \max_{\rho \in T} u_{\rho}(z_i),
			\] so Equation~\eqref{eq:submodular} holds.
			\item In the second case, $u_{\rho_1}(z_i) \geq u_{\bar{\rho}}(z) \geq u_{\rho_2}(z_i)$, so \[
			\max_{\rho \in T \cup \{\rho_1\}} u_{\rho}(z_i) + \max_{\rho \in T \cup \{\rho_2\}} u_{\rho}(z_i) = u_{\rho_1}(z_i)  + u_{\bar{\rho}}(z) = \max_{\rho \in T \cup \{\rho_1, \rho_2\}} u_{\rho}(z_i) + \max_{\rho \in T} u_{\rho}(z_i),
			\] so Equation~\eqref{eq:submodular} holds.
			\item In the third and final case, $u_{\rho_1}(z_i)  \geq u_{\rho_2}(z_i) \geq u_{\bar{\rho}}(z)$, so \[\max_{\rho \in T \cup \{\rho_1\}} u_{\rho}(z_i) + \max_{\rho \in T \cup \{\rho_2\}} u_{\rho}(z_i) = u_{\rho_1}(z_i)  +u_{\rho_2}(z_i)
				\geq u_{\rho_1}(z_i)  + u_{\bar{\rho}}(z)
				= \max_{\rho \in T \cup \{\rho_1, \rho_2\}} u_{\rho}(z_i) + \max_{\rho \in T} u_{\rho}(z_i),
			\] so Equation~\eqref{eq:submodular} holds.
		\end{itemize}
Therefore, the function $U$ is monotone and submodular.
\end{proof}
For any cardinality constraint $\kappa \in \N$, let $\hat{\cP} \subseteq \cP^*$ be the set of $\kappa$ parameter settings that the greedy algorithm selects to optimize the function $U$. Theorem~\ref{thm:submodular} implies that \[\sum_{i = 1}^N \max_{\rho \in \hat{\cP}} u_{\rho}(z_i) \geq \left( 1 - \frac{1}{e}\right)\max_{T \subseteq \cP : |T| \leq \kappa}\sum_{i = 1}^N \max_{\rho \in T} u_{\rho}(z_i).\]

\section{Additional details about experiments}\label{app:experiments}
\paragraph{Branch-and-bound.} We begin with a high-level overview of branch-and-cut (B\&C) and refer the reader to the textbook by~\citet{Nemhauser99:Integer}, for example, for more details. B\&C is an algorithm for solving integer programs (IPs). An IP is defined by an objective vector $\vec{c} \in \R^n$, a constraint matrix $A \in \R^{m \times n}$, a constraint vector $\vec{b} \in \R^m$, and a set of indices $I \subseteq [m]$. The goal is to solve the following optimization problem:
\begin{equation}\begin{array}{lll}
	\text{maximize} & \vec{c} \cdot \vec{x}\\
	\text{subject to} & A \vec{x} \leq \vec{b}\\
	&x[i] \in \Z &\forall i \in I.\end{array}\label{eq:IP}\end{equation}
In keeping with Section~\ref{sec:formulation}, we use the notation $z = (\vec{c}, A, \vec{b}, I)$ to denote the IP. B\&C builds a search tree to solve an input IP $z$, with $z$ stored at the root. It begins by solving the LP relaxation of the input IP $z$. We use the notation $\breve{\vec{x}}_z$ to denote the solution to this LP relaxation. B\&C then uses a \emph{variable selection policy} to choose a variable $i \in I$ and it \emph{branches} on this variable. This means that it defines a new IP $z_i^-$ which is identical to the original IP $z$ but with the additional constraint that $x[i] \leq \left \lfloor \breve{x}_z[i]\right\rfloor$. It stores the IP $z_i^-$ in the left child of the root node. Similarly, it defines another IP $z_i^+$ which is identical to the original IP $z$ but with the additional constraint that $x[i] \geq \left \lfloor \breve{x}_z[i]\right\rfloor$. It stores the IP $z_i^+$ in the right child of the root node. It then uses a \emph{node selection policy} to choose one of the two leaves and repeats this process---solving the LP relaxation of the node's IP, choosing a variable to branch on, and so on. Eventually, one of the solutions to an LP relaxation B\&C solves will in fact be the optimal solution to the original IP (Equation~\eqref{eq:IP}), and B\&C will be able to verify its optimality (this verification procedure is straightforward, but we do not go into the details here).

\paragraph{Parameterized variable selection policy.} We analyze a parameterized variable selection policy that has been studied extensively in prior research~\citep{Gauthier77:Experiments,Benichou71:Experiments,Beale79:Branch,Linderoth99:Computational,Achterberg09:SCIP, Balcan18:Learning, Balcan20:Refined}. To define this variable selection policy, we use the notation $\breve{c}_{\bar{z}} = \vec{c} \cdot \breve{\vec{x}}_{\bar{z}}$ for any IP $\bar{z}$. Given a parameter setting $\rho \in [0,1]$ and the IP $\bar{z}$ contained at the leaf of the search tree, this variable selection policy chooses to branch on the variable $i \in I$ that maximizes \[(1-\rho) \min\left\{\breve{c}_{\bar{z}} - \breve{c}_{\bar{z}_i^+}, \breve{c}_{\bar{z}} - \breve{c}_{\bar{z}_i^-}\right\} + \rho\max\left\{\breve{c}_{\bar{z}} - \breve{c}_{\bar{z}_i^+}, \breve{c}_{\bar{z}} - \breve{c}_{\bar{z}_i^-}\right\}.\] The parameter $\rho$ thus balances a pessimistic approach to branching---which always chooses the variable leading to the minimal change in the LP objective value---with an optimistic approach---which chooses the variable leading to the maximal change in the LP objective value.

\end{document}